\documentclass[12pt]{article}
\usepackage{epsfig}
\usepackage{times}
\usepackage{amsmath,amsthm,amsfonts,amssymb,mathrsfs}
\usepackage[fleqn,tbtags]{mathtools}

\usepackage{algorithm}
\usepackage{algorithmic}

\newcommand{\beq}{\begin{equation}} \newcommand{\eeq}{\end{equation}}

\newcommand{\la}{{\langle}}
\newcommand{\ra}{{\rangle}}
\newcommand{\ka}{{{\kappa}}}
\newcommand{\reg}{{\omega}}
\newcommand{\vp}{{\varphi}}

\newcommand{\sign}{{\rm sign}}

\newcommand{\tgamma}{{\tilde \gamma}}

\newcommand{\bi}{\begin{itemize}}
\newcommand{\be}{\begin{enumerate}}
\newcommand{\ei}{\end{itemize}}
\newcommand{\ee}{\end{enumerate}}

\newcommand{\R}{{\mathbb R}}
\newcommand{\N}{{\mathbb N}}

\newcommand{\lam}{{\lambda}}

\def\boldf#1{\hbox{\rlap{$#1$}\kern.4pt{$#1$}}}

\newcommand{\trans}{^{\scriptscriptstyle \top}}

\mathtoolsset{showonlyrefs,showmanualtags}

\newcommand{\rd}{\R^d}

\DeclareMathOperator\prox{prox}
\DeclareMathOperator\argmin{argmin}
\newcommand{\uh}{\hat{y}}

\begin{document}

\title{Efficient First Order Methods for Linear Composite Regularizers} 

\author{
{\bf Andreas Argyriou} \\
Toyota Technological Institute at Chicago, University of Chicago \\
6045 S. Kenwood Ave.  Chicago, Illinois 60637, USA \and
{\bf Charles A. Micchelli\footnote{Also with Department of Mathematics and Statistics, 
University at Albany, Earth Science 110 Albany, NY 12222, USA.}} \\
Department of Mathematics, City University of Hong Kong \\ 
83 Tat Chee Avenue Kowloon Tong, Hong Kong \and
{\bf  Massimiliano Pontil}\\
Department of Computer Science, University College London \\
Malet Place
London WC1E 6BT, UK \and 
{\bf Lixin Shen} \\
Department of Mathematics, Syracuse University \\
215 Carnegie Hall
Syracuse, NY 13244-1150, USA \and
{\bf Yuesheng Xu} \\
Department of Mathematics, Syracuse University \\
215 Carnegie Hall Syracuse, NY 13244-1150, USA
}

\maketitle

\begin{abstract}
A wide class of regularization problems in machine learning and
statistics employ a regularization term which is obtained by composing
a simple convex function $\omega$ with a linear transformation. This
setting includes Group Lasso methods, the Fused Lasso and other total
variation methods, multi-task learning methods and many more. In this
paper, we present a general approach for computing the proximity
operator of this class of regularizers, under the assumption that the
proximity operator of the function $\omega$ is known in advance. Our
approach builds on a recent line of research on optimal first order
optimization methods and uses fixed point iterations for numerically
computing the proximity operator. It is more general than
current approaches and, as we show with numerical simulations,
computationally more efficient than available first order methods
which do not achieve the optimal rate. In particular, our method
outperforms state of the art $O(\frac{1}{T})$ methods for overlapping Group
Lasso and matches optimal $O(\frac{1}{T^2})$ methods for the Fused Lasso
and tree structured Group Lasso.
\end{abstract}

\newcommand{\bproof}{\begin{proof}}
\newcommand{\eproof}{\end{proof}}
\newenvironment{essential}{\begin{minipage}[t]{11.4cm}
\textwidth=15.2cm\parskip=5mm\vspace{0.2cm}}{\vspace{0.5cm}\end{minipage}}

%%%%%%%%
% BEGIN
\renewcommand{\theequation}{\thesection.\arabic{equation}}
\numberwithin{equation}{section}
\renewcommand{\bar}{\overline}
\newtheorem{theorem}{Theorem}[section]
\newtheorem{question}{Question}[section]
\newtheorem{proposition}{Proposition}[section]
\newtheorem{lemma}{Lemma}[section]
\newtheorem{corollary}{Corollary}[section]
\newtheorem{definition}{Definition}[section]
\newtheorem{problem}{\em Problem}[section]
\newtheorem{remark}{Remark}[section]
\newtheorem{example}{Example}[section]
\newtheorem{case}{Case}[section]
\newtheorem{assumption}{Assumption}[section]

\renewcommand\proofname{\bf Proof} %Change proofname to bold
\def\eop{$\rule{1.3ex}{1.3ex}$}
\renewcommand\qedsymbol\eop  % To end a proof by a black square
\numberwithin{equation}{section}
\makeatletter

%%%%%%%
% Local definitions
%%%%%%%

%%%%%%%%%%%%%%%%%%%%%%%%%%%%%%%%%%%%%%%%%%%%%%%
%%%%%%%%%%%%%%%%%%%%%%%%%%%%%%%%%%%%%%%%%%%%%%%

\section{Introduction}
%the problem
In this paper, we study supervised learning methods which are based on
the optimization problem
\beq
\min_{x \in \R^d} f(x)+g(x)
\label{eq:opt}
\eeq
where the function $f$ measures the fit of a vector $x$ to available
training data and $g$ is a penalty term or regularizer which
encourages certain types of solutions. More precisely we let $f(x)=
E(y,Ax)$, where $E: \R^s \times \R^s \rightarrow [0,\infty)$ is an
error function, $y\in \R^s$ is vector of measurements and $A \in \R^{s
\times d}$ a matrix, whose rows are the input vectors. This class of
regularization methods arise in machine learning, signal processing
and statistics and have a wide range of applications.

%More precisely we let $f(x)=
%\sum_{i\in \N_s} \ell(\la x,x_i\ra, y_i)$, where $\N_s=\{1,\dots,s\}$,
%$\{(x_i,y_i): i \in \N_s\} \subseteq \R^d \times \R$ is a training
%set, $\la\cdot,\cdot\ra$ denotes the standard Euclidean inner product
%on $\R^d$ and $\ell: \R \times \R \rightarrow [0,\infty)$ is a loss
%function. This class of regularization methods arise in machine
%learning, signal processing and statistics and have a wide range of
%applications.
 
%example of the problem
%Different choices of the loss function and the penalty function
Different choices of the error function and the penalty function
correspond to specific methods. In this paper, we are interested in
solving problem \eqref{eq:opt} when $f$ is a {\em strongly smooth
convex} function (such as the square error $E(y,Ax) =\|y-Ax\|_2^2$) and the penalty function
$g$ is obtained as the composition of a ``simple'' function with a
linear transformation $B$, that is,
\beq
g(x) = \omega(Bx)
\label{eq:comp-reg}
\eeq
where $B$ is a prescribed $m \times d$ matrix and $\omega$ is a 
{\em nondifferentiable convex} function on $\R^d$. 
The class of regularizers \eqref{eq:comp-reg}
includes a plethora of methods, depending on the choice of the
function $\omega$ and of matrix $B$. Our motivation for studying
this class of penalty functions arises from sparsity-inducing
regularization methods which consider $\omega$ to be either the
$\ell_1$ norm or a mixed $\ell_1$-$\ell_p$ norm. When $B$ is the
identity matrix and $p=2$, the latter case corresponds to
the well-known Group Lasso method \cite{yuan}, for which well studied optimization
techniques are available. Other choices of the matrix $B$ give rise to
different kinds of Group Lasso with overlapping groups
\cite{Jenatton,binyu}, which have proved to be effective in modeling structured sparse regression problems. 
Further examples can be obtained considering composition with the
$\ell_1$ norm (e.g. this includes the Fused Lasso penalty function
\cite{tib05} and other total variation methods \cite{MSX}) as well as composition
with orthogonally invariant norms, which are relevant, for example, in
the context of multi-task learning \cite{argyriou2010spectral}.

%proximal methods 
A common approach to solve many optimization problems of the 
general form \eqref{eq:opt}
is via proximal methods. These are first-order iterative methods, 
whose computational cost per iteration is comparable to gradient descent.
In some problems in which $g$ has a simple enough form, they can be combined 
with acceleration techniques
\cite{beck09,Nesterov83,Nesterov07,tseng08,tseng10}, to 
yield significant gains in the
number of iterations required to reach a certain approximation 
accuracy of the minimal value. 
The essential step of proximal methods requires the computation of the proximity
operator of function $g$ (see Definition
\ref{def:prox} below). In certain cases of practical importance,
this operator admits a closed form, which makes proximal methods
appealing to use. However, in the general case \eqref{eq:comp-reg}
the proximity operator may not be easily computable.
%Indeed, unless $B=I$, the penalty term cannot be
%decomposed as the sum of univariate functions, and the optimization
%problem is more difficult to solve. 
We are aware of techniques to compute this operator for only some
specific choices of the function $\omega$ and the matrix $B$. Most
related to our work are recent papers for Group Lasso with overlap
\cite{liu2010fast} and Fused Lasso \cite{liu2010}. See also 
\cite{AEP,beck09,kim2010scalable,Obo10,mosci10} for other optimization methods 
for structured sparsity.

%contribution 
The main contribution of this paper is a general technique to compute
the proximity operator of the composite regularizer
\eqref{eq:comp-reg} from the solution of a certain fixed point
problem, which depends on the proximity operator of the function
$\omega$ and the matrix $B$. This fixed point problem can be solved by
a simple and efficient iterative scheme when the proximity operator of
$\omega$ has a closed form or can be computed in a finite number of
steps.  When $f$ is a strongly smooth function, the above result can
be used together with Nesterov's accelerated method
\cite{Nesterov83,Nesterov07} to provide an efficient first-order method for solving
the optimization problem \eqref{eq:opt}.  Thus, our technique allows
for the application of proximal methods on a much wider class of
optimization problems than is currently possible. Our technique is
both more general than current approaches and also, as we argue with
numerical simulations, computationally efficient.  In
particular, we will demonstrate that our method outperforms state of
the art $O(\frac{1}{T})$ methods for overlapping Group Lasso and matches
optimal $O(\frac{1}{T^2})$ methods for the Fused Lasso and tree structured
Group Lasso.
%
%that is we consider a penalty function of the form $\reg \circ B$, where $B$ is
%a matrix and $\reg$ a convex function, whose proximity operator
%is known. In the main body of the paper, we show that the computation
%of the proximity operator of the composition $\reg \circ B$ can be linked to 
%the solution of a certain fixed point problem, 
%which can be solved by Picard iterations. This idea has been recently
%proposed in \cite{MSX}. Our aim here is to extend its
%applicability to convex optimization problems which are relevant in machine
%learning and statistics. In particular, we shall consider Group Lasso
%penalty functions with arbitrary overlapping groups, as well as certain
%composite orthogonal invariant norms.

The paper is organized as follows. In Section \ref{sec:2}, we review
the notion of proximity operator and useful facts from fixed point
theory. In Section \ref{sec:3}, we discuss some examples of composite
functions of the form \eqref{eq:comp-reg} which are valuable in
applications. In Section
\ref{sec:4}, we present our technique to compute the proximity operator for 
a composite regularizer of the form \eqref{eq:comp-reg} and then an
algorithm to solve the associated optimization
problem~\eqref{eq:opt}. In Section \ref{sec:5}, we report our
numerical experience with this method.

%%%%%%%%%%%%%%%%%%%%%%%%%%%%%%%%%%%%%%%%%%%%%%%
%%%%%%%%%%%%%%%%%%%%%%%%%%%%%%%%%%%%%%%%%%%%%%%

\section{Background}
\label{sec:2}

%\massi: I remove the following paragraph
%In this section we recall the definition of proximity operator and
%some of its useful properties.

We denote by $\la\cdot,\cdot\ra$ the Euclidean inner product on $\R^d$
and let $\|\cdot\|_2$ be the induced norm. If $v:\R \rightarrow \R$, for
every $x \in \R^d$ we denote by $v(x)$ the vector $(v(x_i): i \in
\N_d)$, where, for every integer $d$, we use $\N_d$ as a shorthand 
for the set $\{1,\dots,d\}$. For every $p \geq 1$, we define the $\ell_p$ norm of $x$ as $\|x\|_p = 
(\sum_{i\in
\N_d}|x_i|^p)^\frac{1}{p}$.

The proximity operator on a Hilbert space was introduced by Moreau in
\cite{moreau62,moreau65}. 
%For simplicity, throughout the paper we focus on 
%finite dimensional Hilbert spaces, as the general case is not
%needed for our analysis. We refer the reader to the above papers for
%the general case.

\begin{definition}
Let $\reg$ be a real valued convex function on $\rd$. The proximity operator of $\reg$ is defined, for every $x\in\rd$ by
\beq
\prox_\reg (x) := \argmin\limits_{y \in \rd} 
\left\{ \dfrac{1}{2} \|y-x\|_2^2 + \reg(y) \right\} \,.
\label{eq:prox}
\eeq
\label{def:prox}
\end{definition}
\vspace{-.3truecm}
The proximity operator is well defined, because the
above minimum exists and is unique. 
%(see Proposition \ref{prop:prox-welldef}). 

Recall that the subdifferential of a convex function $\reg$ at $x$ is defined as
$$
\partial \reg(x) = \{u: u \in \R^d, \la y-x,u\ra +\reg(x) \leq \reg(y),~ y \in \R^d\}.
$$ The subdifferential is a nonempty compact and convex set. Moreover,
if $\reg$ is differentiable at $x$ then its subdifferential at $x$
consists only of the gradient of $\reg$ at $x$. 
The next proposition establishes a relationship between the proximity
operator and the subdifferential of $\reg$ -- see, for example,
\cite[Prop.~2.6]{MSX} for a proof.
\begin{proposition}
If $\reg$ is a convex function on $\rd$ and $y\in\rd$ then
\beq
x\in\partial\reg(y) ~\quad\text{if and only if}\quad~ y=\prox_\reg(x+y) \,.
\label{eq:grad_prox}
\eeq
\label{prop:grad_prox}
\vspace{-.3truecm}
\end{proposition}
We proceed to discuss some examples of functions $\reg$ and the
corresponding proximity operators. 
%\massi: I removed this phrase:
%Most of these examples can be found
%in \cite{combettes,duchi}. 

If $\reg(x) = \lambda \|x\|_p^p$, where $\lam$ is a positive
parameter, we have that
\beq
\prox_{\reg}(x) = h^{-1}(|x|) {\rm sign}(x)
\label{eq:proxLp}
\eeq
where the function $h: [0,\infty) \rightarrow [0,\infty)$ is defined, for every
$t \geq 0$, as $h(t) = \lambda \, p \, t^{p-1} + t$. 
%Note that $h$ is strictly increasing, hence $h^{-1}$ is well-defined. 
This fact follows immediately from the optimality condition
of the optimization problem \eqref{eq:prox}. Using the above
equation, we may also compute the proximity map of a multiple of the
$\ell_p$ norm, namely the case that $\omega = \gamma \|\cdot\|_p$, where $\gamma > 0$. Indeed, for every $x \in\R^d$, there exists a value of
$\lam$, depending only on $\gamma$ and $x$, such that the optimization
problem \eqref{eq:prox} for $\omega = \gamma \|\cdot\|_p$ equals to
the solution of the same problem for $\omega = \lam
\|\cdot\|_p^p$. Hence the proximity map of the $\ell_p$ norm can be
computed by \eqref{eq:proxLp} together with a simple line search. The
cases that $p\in \{1,2\}$ are simpler, see e.g. \cite{combettes}. For $p=1$ we obtain the
well-known soft-thresholding operator, namely
\beq
\prox_{\lam \|\cdot\|_1} = (|x|-\lam)_+ \sign(x),
\label{eq:proxl1}
\eeq 
where, for every $t \in \R$, we define $(t)_+=t$ if $t \geq 0$ and
zero otherwise; when $p=2$ we have that
\begin{equation}
\prox_{\lam\|\cdot\|_2}(x) = 
\left\{
\begin{array}{ll}
(\|x\|_2-\lam)_+ \frac{x}{\|x\|_2} & \text{if~} x \neq 0 \\
0 & \text{if~}x=0.
\end{array} \right.
\label{eq:proxL2}
\end{equation}
In our last example, we consider the $\ell_\infty$ norm, which is defined,
for every $x \in \R^d$ as $\|x\|_\infty = \max\{|x_i|: i \in
\N_d\}$. We have that $$
\prox_{\lam \|\cdot\|_\infty}(x) = \min\left\{|x|,
\frac{1}{k} \sum_{|x_i| > s_k} |x_i|-\lam
\right\}
\sign(x)
$$
where $s_k$ is the $k$-th largest value of the components of the vector $|x|$ and 
$k$ is the largest integer such that $\sum_{|x_i| > s_k} (|x|_i-s_k) < \lam$. For a proof of the above 
formula, see, for example \cite[Sec.~5.4]{duchi}. 

Finally, we recall some basic facts about fixed point theory which are
useful for our study. For more information on the material presented
here, we refer the reader to \cite{zalinescu}. 

%A mapping $\vp: \R^d
%\rightarrow \R^d$ is called {\em firmly nonexpansive} if $$
%\|\vp(x)-\vp(y)\|^2_2 \leq \la \vp(x)-\vp(y),x-y\ra
%$$ 
%and {\em nonexpansive} if
%$$
%\|\vp(x)-\vp(y)\|_2 \leq \| \vp(x)-\vp(y)\|_2.
%$$ 
%It follows by an application of the Cauchy-Schwarz inequality, that
%if $\vp$ is firmly nonexpansive then both $\vp$ and $I - \vp$ are
%nonexpansive. 

A mapping $\vp: \R^d \rightarrow \R^d$ is called strictly
non-expansive (or contractive) if there exists $\beta \in [0,1)$ such
that, for every $x,y\in \R^d$, $\|\vp(x)-\vp(y)\|_2\leq \beta \|x-y\|_2$. 
If the above inequality holds for $\beta =1$, the mapping is called
nonexpansive. 
As noted in \cite[Lemma 2.4]{combettes}, 
%the proximity operator is firmly nonexpansive.
both $\prox_{\reg}$ and $I- \prox_{\reg}$ are nonexpansive. 

We say that $x$ is a {\em fixed point} of a mapping $\vp$
if $x=\vp(x)$. The Picard iterates $x^n, n \in \N$, starting at $x_0
\in \R^d$ are defined by the recursive equation $x^{n} = \vp(x^{n-1})$.
It is a well-known fact that, if $\vp$ is strictly nonexpansive then
$\vp$ has a unique fixed point $x$ and $\lim_{n\rightarrow \infty} x^n
= x$. However, this result fails if $\vp$ is nonexpansive. We end this
section by stating the main tool which we use to find a fixed point of
a nonexpansive mapping $\vp$.
\begin{theorem} (Opial $\ka$-average theorem \cite{opial})
Let $\vp: \R^d \rightarrow \R^d$ be a nonexpansive mapping, which has
at least one fixed point and let $\vp_\ka := \ka I + (1-\ka) \vp$.
Then, for every $\ka \in (0,1)$, the Picard iterates of $\vp_\ka$
converge to a fixed point of $\vp$.
\label{thm:opial}
\end{theorem}

\section{Examples of Composite Functions}
\label{sec:3}
In this section, we show that several examples of penalty functions which have appeared in the literature 
fall within the class of linear composite functions \eqref{eq:comp-reg}. 

%%%%%%%% Example 1: Group Lasso %%%%%%%%
%%
We define for every $d \in \N$, $x \in \R^d$ and $J \subseteq \N_d$, the 
restriction of the vector $x$ to the index set $J$ as $x_{|J} = (x_i: i \in J)$. 
Our first example considers the Group Lasso penalty function, 
which is defined as
\beq
\reg_{\rm GL}(x) = \sum_{\ell \in \N_k} \|x_{|J_\ell}\|_2
\label{eq:GL}
\eeq 
where $J_\ell$ are prescribed subsets of $\N_d$ (also
called the ``groups'') such that $\cup_{\ell=1}^k J_\ell = \N_d$. The
standard Group Lasso penalty (see e.g. \cite{yuan}) corresponds to the
case that the collection of groups $\{J_\ell : \ell \in \N_k\}$ forms
a partition of the index set $\N_d$, that is, the groups do not
overlap. In this case, the optimization problem \eqref{eq:prox} for
$\omega=\omega_{\rm GL}$ decomposes as the sum of separate problems
and the proximity operator is readily obtained by applying the formula
\eqref{eq:proxL2} to each group separately.  In many cases of
interest, however, the groups overlap and the proximity operator
cannot be easily computed.
% andy  
% For example, this is the case when the
% groups are chosen in a hierarchical fashion, see
% \cite{Jenatton,kim2010scalable,kim09}, and references therein.

Note that the function \eqref{eq:GL} is of the form \eqref{eq:comp-reg}. 
We let $d_\ell = |J_\ell|$, $m =\sum_{\ell \in \N_k} d_\ell$ and define, for 
every $z \in \R^m$, $\reg(z) = \sum_{\ell \in \N_k} \|z_{\ell}\|_2$, 
where, for every $\ell \in \N_k$ we let $z_\ell = (z_i: \sum_{j\in \N_{\ell-1}} 
d_j< i \leq \sum_{j\in \N_\ell} d_j)$.
Moreover, we choose $B = [B_1\trans,\dots,B_k\trans]\trans$, 
where $B_\ell$ is a $d_\ell \times d$ matrix defined as
\begin{equation*}
(B_\ell)_{ij} = \left\{
\begin{array}{rl}
1 & \text{if~} j = J_\ell[i] \\
0 & \text{otherwise}
\end{array} \right.
\end{equation*}
where for every $J \subseteq \N_d$ and $i \in \N_{|J|}$, we denote by $J[i]$ 
the $i$-th largest integer in $J$. 
%That is, we choose
%\beq
%\reg(Bx) = \sum_{\ell=1}^k h(B_k x).
%\label{eq:comp1}
%\eeq
%where the function $h:\R^d \rightarrow \R$ is assumed to be well defined
%and convex for every $d \in \N$. If, in addition, $h$
%is a norm, then $\reg$ is a norm as well. For example, we may choose
%$h$ to be the $\ell_2$ norm. In this case, regularization with the
%penalty function \eqref{eq:mix} is known as the Group Lasso method \cite{yuan}.
%problem is more difficult to solve. 
%Indeed, unless $B=I$, the penalty term cannot be
%decomposed as the sum of univariate functions, and the optimization
%problem is more difficult to solve. 

%%%%%%%% Example 2: Fused Lasso %%%%%%%%
%%
The second example concerns the Fused Lasso \cite{tib05}, which considers 
the penalty function $x \mapsto g(x)= \sum_{i\in \N_{d-1}} |x_{i}-x_{i+1}|$. 
It immediately follows that this function falls into the class \eqref{eq:comp-reg} 
if we choose $\reg$ to be the $\ell_1$ norm and $B$ the first order divided difference
matrix
\begin{equation}
B= \left[
\begin{array}{rrrrr}
1 & -1 & 0 & \ldots & \ldots\\
0 & 1 & - 1 & 0 & \ldots \\
\vdots & \ddots & \ddots & \ddots & \ddots\\
\end{array} \right].
\label{eq:fused}
\end{equation}
The intuition behind the Fused Lasso is that it favors vectors which
do not vary much across contiguous components.  Further extensions of
this case may be obtained by choosing $B$ to be the incidence matrix
of a graph, a setting which is relevant for example in online learning
over graphs \cite{mark09}. Other related examples include the
anisotropic total variation, see for example, \cite{MSX}.

%%%%%%%% Example 3: Composed OI norms %%%%%%%%
%%
The next example considers composition with orthogonally invariant (OI) 
norms. Specifically, we choose a symmetric gauge function $h$,
that is, a norm $h$, which is both {\em absolute} and {\em invariant under
permutations} \cite{von-neumann} and define the function
$\reg:\R^{d\times n} \rightarrow [0,\infty)$, at $X$ by the formula
$$
\reg(X) = h(\sigma(X))
$$ 
where $\sigma(X) \in [0,\infty)^{r}$, $r = \min(d,n)$ is the vector
formed by the singular values of matrix $X$, in non-increasing
order. An example of OI-norm are Schatten
$p$-norms, which correspond to the case that $\reg$ is
the $\ell_p$-norm. The next proposition provides a formula for the proximity operator of 
an OI-norm. The proof is based on an inequality by von Neumann
\cite{von-neumann}, sometimes called von Neumann's trace
theorem or Ky Fan's inequality.
%We postpone it for another occasion since it is not central to our development.

\begin{proposition}
With the above notation, it holds that
$$
\prox_{h \circ \sigma}(X) = U {\rm diag}\left(\prox_h(\sigma(X))\right) V\trans
$$ 
where $X=U{\rm diag}(\sigma(X)) V\trans$ and $U$ and $V$ are the matrices 
formed by the left and right singular vectors of $X$, respectively. 
%XXX Massi: changed a bit the above
\label{prop:prox-OI}
\end{proposition}
\begin{proof}
The proof is based on an inequality by von Neumann
\cite{von-neumann}, sometimes called von Neumann's trace
theorem or Ky Fan's inequality. It states that $\la X,Y \ra \leq \la
\sigma(X),\sigma(Y)\ra$, with equality if and only if $X$ and $Y$ share
the same ordered system of singular vectors. Note that
\begin{eqnarray}
\|X-Y\|^2_2  & = &  \|X\|_2^2 + \|Y\|_2^2 - 
2 \la X,Y\ra\\ \nonumber 
~ & \geq &  \|\sigma(X)\|_2^2 + \|\sigma(Y)\|_2^2 - 2 \la \sigma(X),\sigma(Y)\ra \\ \nonumber
~ &=& \|\sigma(X) - \sigma(Y)\|^2_2
%\\ \nonumber 
%& = &  \|\sigma(X) - \sigma(Y)\|^2_2 
\end{eqnarray}
and the equality holds if and only if
$Y=U {\rm diag}( \sigma(Y) ) V\trans$. 
Consequently, we have that
\begin{eqnarray}
\frac{1}{2}\|X-Y\|^2_2  + \reg(Y) & \geq & \frac{1}{2}
\|\sigma(X) -\prox_h(\sigma(X))\|_2^2 \\ \nonumber
& ~ & + h(\prox_h(\sigma(X))) \,.
\end{eqnarray}
To conclude the proof we need to show that
$\gamma := \prox_{h}(\sigma(X))$ has the same ordering of $\sigma$,
that is, $\gamma$ is non-increasing. Suppose on the contrary that there exists
$i,j \in
\N_d$, $i < j$, such that $\gamma_i < \gamma_j$. Let $\tgamma$ be the
vector obtained by flipping the $i$-th and $j$-th components of
$\gamma$. A direct computation gives 
$$
\frac{1}{2}\|\sigma-\gamma\|^2_2+ h(\gamma) - 
\frac{1}{2}\|\sigma-\tgamma\|^2_2 - h(\tgamma) = (\sigma_i-\sigma_j)(\gamma_i-\gamma_j).
$$
Since the left hand side of the above equation is positive, this leads to a contradiction.
\end{proof}

We can compose an OI-norm with a linear transformation $B$, this time
between two spaces of matrices, obtaining yet another subclass of
penalty functions of the form
\eqref{eq:comp-reg}.  This setting is relevant in the context of
multi-task learning. For example  
% \cite{EMP}, chooses $h$ to be
%the $\ell_2$ norm and considers several specific linear
%transformations which model task relatedness, while 
\cite{Evg06} chooses $h$ to be
the {\em trace} or {\em nuclear} norm and considers a specific linear
transformation which model task relatedness, namely,
that $g(X) = \left\|\sigma\left(X(I-\frac{1}{n}ee\trans)\right)\right\|_1$, 
where $e \in \R^d$ is the vector all of whose components are equal to one.

%Our final example corresponds to choosing the convex function 
%defined, for every $z \in \R^m$, by the formula
%\begin{equation*}
%\omega(z;J)= \|z_{J}\| + \delta_\{\|z_{|{\bar J}}\|_\infty \leq 1\} 
%\end{equation*}
%where $J$ is a prescribe subset of $\N_m$, ${\bar J} = \N_m \backslash J$.
%Composition with this penalty function may be useful, for example, in
%image coding. We have a dictionary matrix $A$ and wish to find a
%sparse coefficient vector $x$, such that $Ax$ approximates well an
%image $y$ and satisfies the constraint that $\|Ax\|_\infty \leq 1$.
%This idea may be implemented by solving the optimization problem
%\eqref{eq:opt} for $f(x) = \|y-Ax\|_2^2$, $g(x) = \lambda
%\omega(Bx;J)$, where $m=2d$, $J= \N_d$ and $B = [I,A\trans]\trans$.

%%%%%%%%%%%%%%%%%%%%%%%%%%%%%%%%%%%%%%%%%%%%%%
%%%%%%%%%%%%%%%%%%%%%%%%%%%%%%%%%%%%%%%%%%%%%%

\section{Fixed Point Algorithms Based on Proximity Operators}
\label{sec:4}
We now propose optimization approaches which use fixed
point algorithms for nonsmooth problems. 
We shall focus on problem \eqref{eq:opt}
under the assumption \eqref{eq:comp-reg}.
We assume that $f$ is a {\em strongly smooth}
convex function, that is, $\nabla f$ is Lipschitz continuous
with constant $L$,
and $\omega$ is a {\em nondifferentiable} convex function.
A typical class of such problems occurs in regularization
methods where $f$ corresponds to a data error term with, say, 
the square loss. 
Our approach builds on proximal methods and uses fixed point 
(also known as Picard) iterations
for numerically computing the proximity operator. 

%%%%%%%%%%%%%%%%%

\subsection{Computation of a Generalized Proximity Operator with a Fixed Point Method}
\label{sec:quad}

As the basic building block of our methods, we consider the optimization problem \eqref{eq:opt}
in the special case when $f$ is a quadratic function, that is,
\beq
\min\left\{ \dfrac{1}{2}y\trans Q y - x\trans y + \reg(By) : y \in\rd \right\} \,.
\label{eq:quad}
\eeq
where $x$ is a given vector in $\rd$ and $Q$ a positive definite $d
\times d$ matrix.

Recall the {\em proximity operator} in Definition
\ref{def:prox}.  Under the assumption that we can explicitly or in a
finite number of steps compute the proximity operator of $\reg$, our
aim is to develop an algorithm for evaluating a minimizer of problem
\eqref{eq:quad}. We describe the algorithm for a generic Hessian $Q$,
as it can be applied in various contexts. For example, it could
lead to a second-order method for solving \eqref{eq:opt}, which
will be the topic of future work.  
In this paper, we will apply the technique to the task of evaluating $\prox_{\reg\circ B}$.

First, we observe that the minimizer of \eqref{eq:quad} exists and is {\em unique}. 
Let us call this minimizer $\uh$.
Similar to Proposition \ref{prop:grad_prox}, we have the following proposition.
\begin{proposition}
If $\reg$ is a convex function on $\R^m$, $Q$ a $d \times d$ positive
definite matrix and $x\in\rd$ then $\uh$ is the solution of problem 
\eqref{eq:quad} {\em if and only if}
\beq
Q\uh \in x-\partial (\reg\circ B)(\uh).
\label{eq:uhat}
\eeq
\label{prop:grad_quad}
\end{proposition}
\vspace{-.3truecm}
The subdifferential $\partial (\reg\circ B)$ appearing in the inclusion \eqref{eq:uhat} 
can be expressed with the chain rule
(see, e.g. \cite{borwein}), which gives the formula
\beq
\partial (\reg\circ B) = B\trans \circ (\partial\reg) \circ B \,.
\label{eq:chain}
\eeq
Combining equations \eqref{eq:uhat} and \eqref{eq:chain} yields the fact that
\beq
Q\uh \in x - B\trans \partial \reg( B\uh) \,.
\label{eq:grad}
\eeq
This inclusion along with Proposition \ref{prop:grad_prox} allows us to express
$\uh$ in terms of the proximity operator of $\reg$. To formulate our observation
we introduce the affine transformation $A:\R^m \to \R^m$ defined, for fixed $x\in\rd$, $\lambda>0$,
at $z\in\R^m$ by
\beq
Az := (I-\lambda BQ^{-1}B\trans) z + BQ^{-1}x
\label{eq:A}
\eeq
and the operator $H:\R^m\to \R^m$ 
\beq
H := \left(I-\prox_{\frac{\reg}{\lam}} \right)\circ A\,.
\label{eq:H}
\eeq

\begin{theorem}
If $\reg$ is a convex function on $\R^m$, $B\in\R^{m \times d}$, $x\in\rd$, $\lambda$
is a positive number and $\uh$ is the minimizer of \eqref{eq:quad} then
\beq
\uh = Q^{-1}(x - \lambda B\trans v)
\label{eq:fixed}
\eeq
if and only if $v\in\R^m$ is a fixed point of $H$.
\label{thm:fixed}
\end{theorem}
\begin{proof}
From \eqref{eq:grad} we conclude that $\uh$ is characterized by the fact that
$\uh = Q^{-1}(x- \lambda B\trans v)$, 
%\beq
%\uh = Q^{-1}(x- \lambda B\trans v)
%\label{eq:fixed1}
%\eeq
where $v$ is a vector in the set $\partial \left(\frac{\reg}{\lambda}\right) (B\uh)$.
Thus it follows that $v \in \partial \left(\frac{\reg}{\lam}\right) \left(B Q^{-1} (x- \lambda B\trans v)\right)$. 
Using Proposition \ref{prop:grad_prox} we conclude that
\beq
B Q^{-1} (x- \lambda B\trans v) = 
%\prox_{\frac{\reg}{\lam}}(BQ^{-1}x + (I - \lambda B Q^{-1} B\trans) v).
\prox_{\frac{\reg}{\lam}}(Av).
\label{eq:fixed2}
\eeq
Adding and subtracting $v$ on the left hand side and rearranging the terms we see that $v$ is a fixed point of $H$.

Conversely, if $v$ is a fixed point of $H$, then equation \eqref{eq:fixed2} holds.
Using again Proposition \ref{prop:grad_prox} and the chain rule \eqref{eq:chain}, we conclude that
\beq
\lambda B\trans v \in\partial\left(\reg\circ B\right) ( Q^{-1} (x-\lambda B\trans v) )
\label{eq:fixed3}
\eeq
Proposition \ref{prop:grad_quad} together with the above inclusion now implies 
that $Q^{-1} (x- \lambda B\trans v)$ is the minimizer of \eqref{eq:quad}.
\end{proof}

Since the operator $(I-\prox_{\frac{\reg}{\lambda}})$ is nonexpansive
\cite[Lemma 2.1]{combettes}, then
\begin{eqnarray}
\|H(v) - H(w)\|_2 & \leq & \|Av - Aw\|_2 \\ \nonumber
~& \leq & \|I-\lam B Q^{-1} B\trans\|\, \|v-w\|_2. 
\end{eqnarray}
We conclude that the mapping $H$ is nonexpansive if the spectral
norm of the matrix $I-\lam BQ^{-1}B\trans$ is not greater than
one. Let us denote by $\lam_j,$ $j \in \N_m$, the eigenvalues of matrix
$BQ^{-1}B\trans$. We see that $H$ is nonexpansive provided that
$|1-\lam \lam_j| \leq 1$, that is if $0\leq \lam \leq 2/\lam_{\rm max}$, 
where $\lam_{\rm max}$ is the spectral norm of $BQ^{-1}B\trans$. 
In this case we can appeal to Opial's Theorem \ref{thm:opial} to find a fixed point of $H$. 

Note that if, for every $j \in \N_m$, $\lam_j > 0$, that is, the
matrix $B Q^{-1} B\trans$ is invertible, then the mapping $H$ is
strictly nonexpansive when $0<\lam < 2/\lam_{\rm max}$. In this case,
the Picard iterates of $H$ converge to the unique fixed point of $H$,
without the need to use Opial's Theorem.

We end this section by noting that, when $Q=I$, the above theorem
provides an algorithm for computing the proximity operator of $\reg \circ B$.
\begin{corollary}
Let $\reg$ be a convex function on $\R^m$, $B\in\R^{m \times d}$, $x\in
\R^d$, $\lambda$ a positive number and define the mapping $v \mapsto
(I - \prox_{\frac{\reg}{\lambda}}) ((I-\lam BB\trans)v + Bx)$. Then
\beq
\prox_{\reg\circ B}(x)= x - \lambda B\trans v
\label{eq:prox-fixed}
\eeq
{\em if and only if} $v$ is a fixed point of $H$.
\label{cor:fixed}
\end{corollary}
Thus, a fixed point iterative scheme like the above one can be used
as part of any proximal method when the regularizer has the form
\eqref{eq:comp-reg}.

%%%%%%%%%%%%%%%%%%%%%%%%%%%%%%%%%%%%%%%%%%%%%%
%%%%%%%%%%%%%%%%%%%%%%%%%%%%%%%%%%%%%%%%%%%%%%

\subsection{Accelerated First-Order Methods}
\label{sec:first}

Corollary \ref{cor:fixed} motivates a general proximal numerical approach to
solving problem \eqref{eq:opt} (Algorithm \ref{alg:prox}).
Recall that $L$ is the
Lipschitz constant of $\nabla f$.
%\massi: I modified the phrase below
%The idea behind proximal methods -- see \cite{combettes}, \cite[ISTA]{fista} etc. --
The idea behind proximal methods -- see
\cite{combettes,fista,Nesterov07,tseng08,tseng10} and references
therein -- is to update the current
estimate of the solution $x_t$ using the proximity operator. 
This is equivalent to replacing $f$ with its linear
approximation around a point $\alpha_t$ specific to iteration $t$. 
The point $\alpha_t$ may depend on the
current and previous estimates of the solution $x_t, x_{t-1}, \dots$,
the simplest and most common update rule being $\alpha_{t}=x_{t}$.

\begin{algorithm}
\caption{Proximal \& fixed point algorithm.}
\begin{algorithmic}
\STATE $x_1,\alpha_1 \leftarrow 0$
\FOR {t=1,2,\dots}
\STATE Compute $x_{t+1} \leftarrow 
\prox_{\frac{\reg}{L}\circ B}\left(\alpha_t - \frac{1}{L}\nabla f(\alpha_t)\right)$
\\\qquad by the Picard-Opial process %andy
\STATE Update $\alpha_{t+1}$ as a function of $x_{t+1},x_t,\dots$
\ENDFOR 
\end{algorithmic}
\label{alg:prox}
\end{algorithm}
%This yields a generic algorithm which can be used with any proximal
%scheme for the update of $\alpha_{t+1}$. 
In particular, in this paper
we focus on combining Picard iterations with {\em accelerated
first-order methods} proposed by Nesterov
\cite{nesterov2005smooth,Nesterov07}. These methods use an $\alpha$
update of a specific type, which requires two levels of memory of $x$.
Such a scheme has the property of a quadratic decay in terms of the
iteration count, that is, the distance of the objective from the
minimal value is $O\left(\frac{1}{T^2}\right)$ after $T$ iterations.
This rate of convergence is optimal for a first order method in the 
sense of the algorithmic model of \cite{nemirovski}.

It is important to note that other methods may achieve faster rates,
at least under certain conditions. For example, {\em interior point methods}
\cite{ip} or {\em iterated reweighted least squares} \cite{daubechies10,osborne85,AEP}
have been applied successfully to nonsmooth convex problems.
However, the former require the Hessian and typically have high cost per iteration. 
The latter require solving linear systems at each iteration. 
Accelerated methods, on the other hand,
have a lower cost per iteration and scale to larger problem sizes.
Moreover, in applications where some type of thresholding operator
is involved -- for example, the Lasso \eqref{eq:proxl1} --
the zeros in the solution are exact, which may be desirable. 

Since their introduction, accelerated methods have quickly become
popular in various areas of applications, including machine learning, 
see, for example, 
\cite{mosci10,lin2010smoothing,liu2010fast,hierarchical} 
and references therein. However, their applicability has been
restricted by the fact that they require {\em exact} computation of
the proximity operator.  Only then is the quadratic convergence rate
known to hold, and thus methods using numerical computation of the
proximity operator are not guaranteed to exhibit this rate.  What we
show here, is how to further extend the scope of accelerated methods
and that, empirically at least, these new methods outperform current
$O\left(\frac{1}{T}\right)$ methods while matching the performance
of optimal $O(\frac{1}{T^2})$ methods. 

In Algorithm \ref{alg:acc} we describe a version of accelerated methods
influenced by \cite{tseng08,tseng10}. Nesterov's insight was that an appropriate
update of $\alpha_t$ which uses two levels of memory 
achieves the $O\left(\frac{1}{T^2}\right)$ rate.
Specifically, the optimal update is 
$\alpha_{t+1} \leftarrow x_{t+1} + \theta_{t+1}\left(\frac{1}{\theta_t}-1\right) (x_{t+1}-x_t)$
where the sequence $\theta_t$ is defined by $\theta_1 = 1$ 
and the recursive equation
\beq
\frac{1-\theta_{t+1}}{\theta_{t+1}^2} = \frac{1}{\theta_t^2} \,.
\label{eq:theta}
\eeq
We have adapted \cite[Algorithm 2]{tseng08} 
(equivalent to FISTA \cite{fista}) by computing the proximity operator of
$\frac{\reg}{L}\circ B$ using the Picard-Opial process described in
Section \ref{sec:quad}.
We rephrased the algorithm using the sequence 
$\rho_t := 1-\theta_t + \sqrt{1-\theta_t}
= 1-\theta_t +\frac{\theta_t}{\theta_{t-1}}$ for numerical 
stability.
At each iteration, the map $A_t$ is defined by 
$$
A_t z := \left(I-\frac{\lambda}{L} BB\trans\right) z -
\frac{1}{L}B(\nabla f(\alpha_t) - L\alpha_t)
$$ 
and $H_t$ as in \eqref{eq:H}.  
By Theorem \ref{thm:fixed}, the fixed
point process combined with the $x$ update are equivalent to $x_{t+1}
\leftarrow \prox_{\frac{\reg}{L}\circ B}\left(\alpha_t -
  \frac{1}{L}\nabla f(\alpha_t)\right)$.

\begin{algorithm}
\caption{Accelerated \& fixed point algorithm.}
\begin{algorithmic}
\STATE $x_1,\alpha_1 \leftarrow 0$
\FOR {t=1,2,\dots}
\STATE Compute a fixed point $v$ of $H_t$ by Picard-Opial %andy
\STATE $x_{t+1} \leftarrow \alpha_t - \frac{1}{L}\nabla f(\alpha_t)
-\frac{\lambda}{L} B\trans v$
\STATE $\alpha_{t+1} \leftarrow \rho_{t+1} x_{t+1} - (\rho_{t+1}-1)x_t$
\ENDFOR 
\end{algorithmic}
\label{alg:acc}
\end{algorithm}

%%%%%%%%%%%%%%%%%%%%%%%%%%%%%%%%%%%%%%%%%%%%%%
%%%%%%%%%%%%%%%%%%%%%%%%%%%%%%%%%%%%%%%%%%%%%%

\section{Numerical Simulations}
\label{sec:5}

We have evaluated the efficiency of our method with simulations on
different nonsmooth learning problems. One important aim of the
experiments is to demonstrate improvement over a % andy
state of the art suite of methods (SLEP) \cite{Liu:2009:SLEP:manual}
in the cases when the proximity operator is not exactly computable. 

An example of such cases which we considered in Section \ref{sec:5.1}
is the Group Lasso with {\em overlapping groups}. An algorithm for
computation of the proximity operator in a finite number of steps is
known only in the special case of hierarchy-induced groups
\cite{hierarchical}. In other cases such as groups induced by
directed acyclic graphs \cite{binyu} or more complicated sets of
groups, the best known theoretical rate for a first-order method is
$O\left(\frac{1}{T}\right)$.  We demonstrate that such a method can be
improved. %andy

Moreover, in Section \ref{sec:5.2} we report efficient convergence in
the case of a composite $\ell_1$ penalty
used for graph prediction \cite{mark09}.
In this case, matrix $B$ is the incidence matrix of a graph
and the penalty is $\sum\limits_{(i,j)\in E}\|x_i-x_j\|_1$,
where $E$ is the set of edges. Most work 
we are aware of for the composite $\ell_1$ penalty
applies to the special cases of total variation
\cite{beck09} or Fused lasso \cite{liu2010}, 
in which $B$ has a simple structure.
A recent method for the general case \cite{becker_candes}
which builds on Nesterov's $O\left(\frac{1}{T}\right)$
smoothing technique \cite{nesterov2005smooth}
does not have publicly available software yet. 

Another advantage of Algorithm \ref{alg:acc} which we
highlight is the high efficiency of Picard iterations
for computing different proximity operators. This 
requires only a small number of iterations regardless
of the size of the problem. 
We also report a roughly linear scalability with
respect to the dimensionality of the problem, which shows
that our methodology can be applied to large scale problems.

In the following simulations, we have chosen the parameter from
Opial's theorem $\kappa=0.2$.  The parameter $\lambda$ was set equal
to $\frac{2L}{ \lambda_{\max} + \lambda_{\min}}$, where $\lam_{\rm
  max}$ and $\lam_{\rm min}$ are the largest and smallest eigenvalues,
respectively, of $\frac{1}{L}BB\trans$.  We have focused exclusively
on the case of the square loss and we have computed $L$ using singular
value decomposition (if this were not possible, a Frobenius estimate
could be used).  Finally, the implementation ran on a 16GB memory dual
core Intel machine. The Matlab code is available at
\texttt{http://ttic.uchicago.edu/$\sim$argyriou/code/\\index.html}.

\begin{figure}[t]
\begin{center}
\includegraphics[width=0.4\textwidth,height=0.3\textwidth]{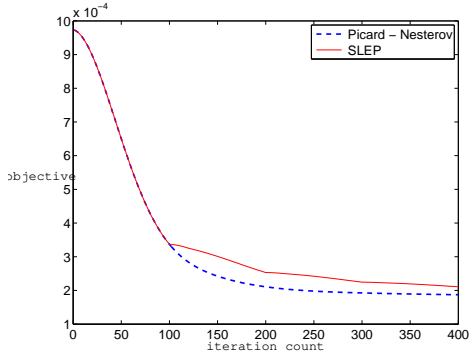}
\caption{Objective function vs. iteration for the overlapping groups data ($d=3500$).
Note that Picard-Nesterov terminates earlier within $\varepsilon$.} 
\label{fig:costs}
\end{center}
\end{figure}

\begin{figure}[t]
\begin{center}
\includegraphics[width=0.4\textwidth]{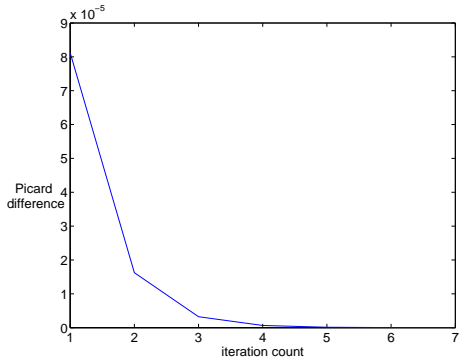}
\caption{$\ell_2$ difference of successive Picard iterates vs. Picard iteration
for the overlapping groups data ($d=3500$).}
\label{fig:fp}
\end{center}
\end{figure}

\subsection{Overlapping Groups}
\label{sec:5.1}
In the first simulation we considered a synthetic
data set which involves a fairly simple group topology
which, however, cannot be embedded as a hierarchy.
We generated data $A\in\R^{s\times d}$, with $s=[0.7d]$ from 
a uniform distribution and normalized the matrix. 
%XXX: how do you normalize it?  the standard way is that 
%each columns (corrsponding to one variable) has norm 1
The target vector $x^*$ was also generated randomly 
so that only $21$ of its components are nonzero.
The groups used in the regularizer $\reg_{GL}$ 
-- see eq. \eqref{eq:GL} -- are: $\{1,...,5\}$, $\{5,...,9\}$,
$\{9,...,13\}$, $\{13,...,17\}$, $\{17,...,21\}$, $\{4,22,...,30\}$,
$\{8,31,...,40\}$, $\{12,41,...,50\}$, $\{16,51,...,60\}$, 
$\{20,61,...,70\}$, $\{71,...,80\}, \dots, \{d-9,...,d\}$.

%\begin{eqnarray}
%\{1,...,5\}, \{5,...,9\},\{9,...,13\}, \{13,...,17\},\{17,...,21\}, &
%\\ \nonumber
%\{4,22,...,30\},\{8,31,...,40\},\{12,41,...,50\}, &\\ \nonumber
%\{16,51,...,60\}, \{20,61,...,70\}, &\\ \nonumber
%\{71,...,80\}, \dots, \{d-9,...,d\} \,.&
%\end{eqnarray}

That is, the first $5$ groups form a chain,
the next $5$ groups have a common element
with one of the first groups and the rest
have no overlaps.
An issue with overlapping group norms
is the coefficients assigned to each group
(see \cite{Jenatton} for a discussion).
We chose to use a coefficient of $1$
for every group and compensate by 
normalizing each component of
$x^*$ according to the number of groups
in which it appears (this of course can only be done
in a synthetic setting like this).
The outputs were then generated as
$y = Ax^* + ~\text{noise}$ with
zero mean Gaussian noise of standard deviation
$0.001$.

\begin{figure}[t]
\begin{center}
\includegraphics[width=0.4\textwidth]{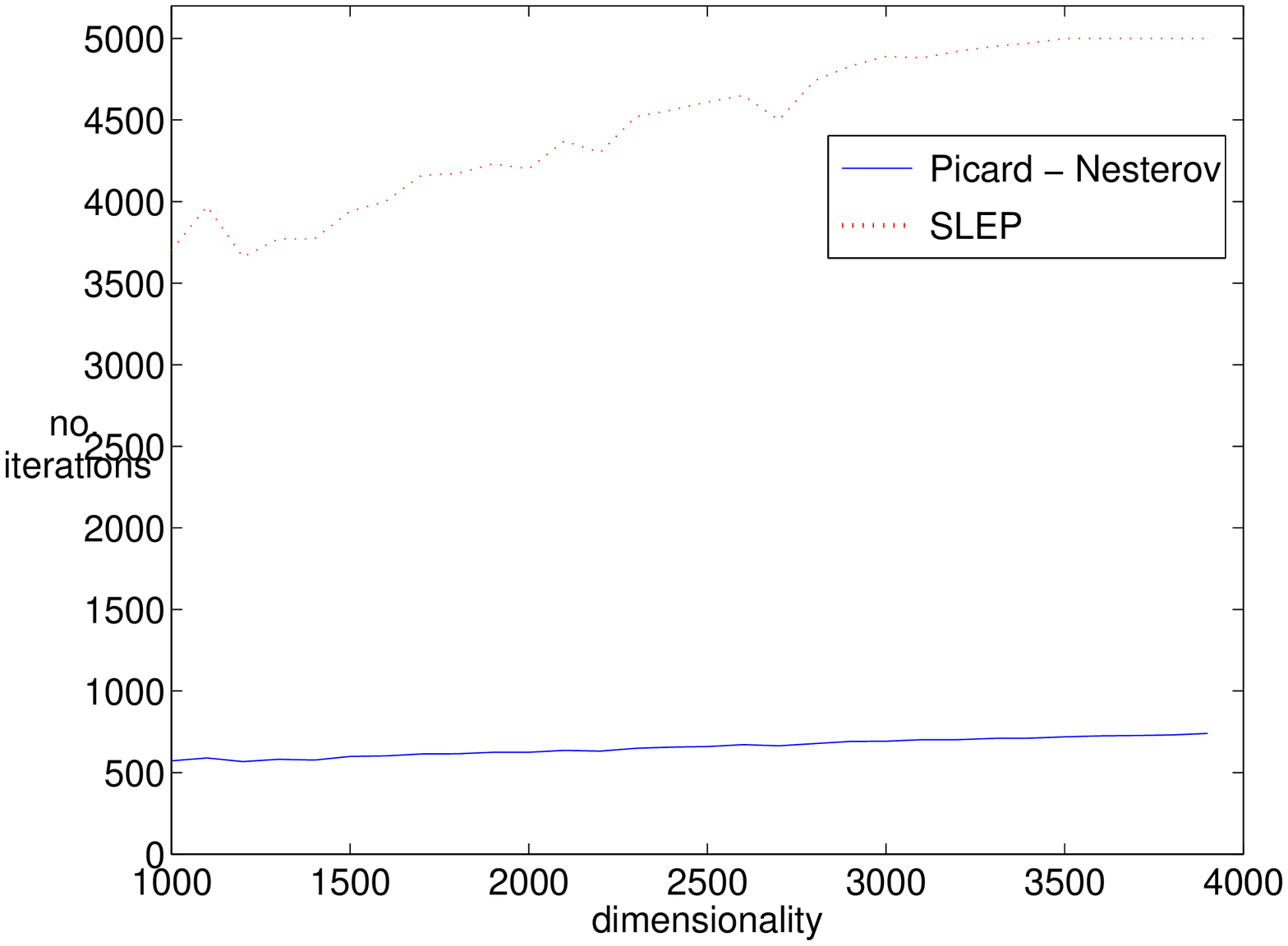}
\includegraphics[width=0.4\textwidth]{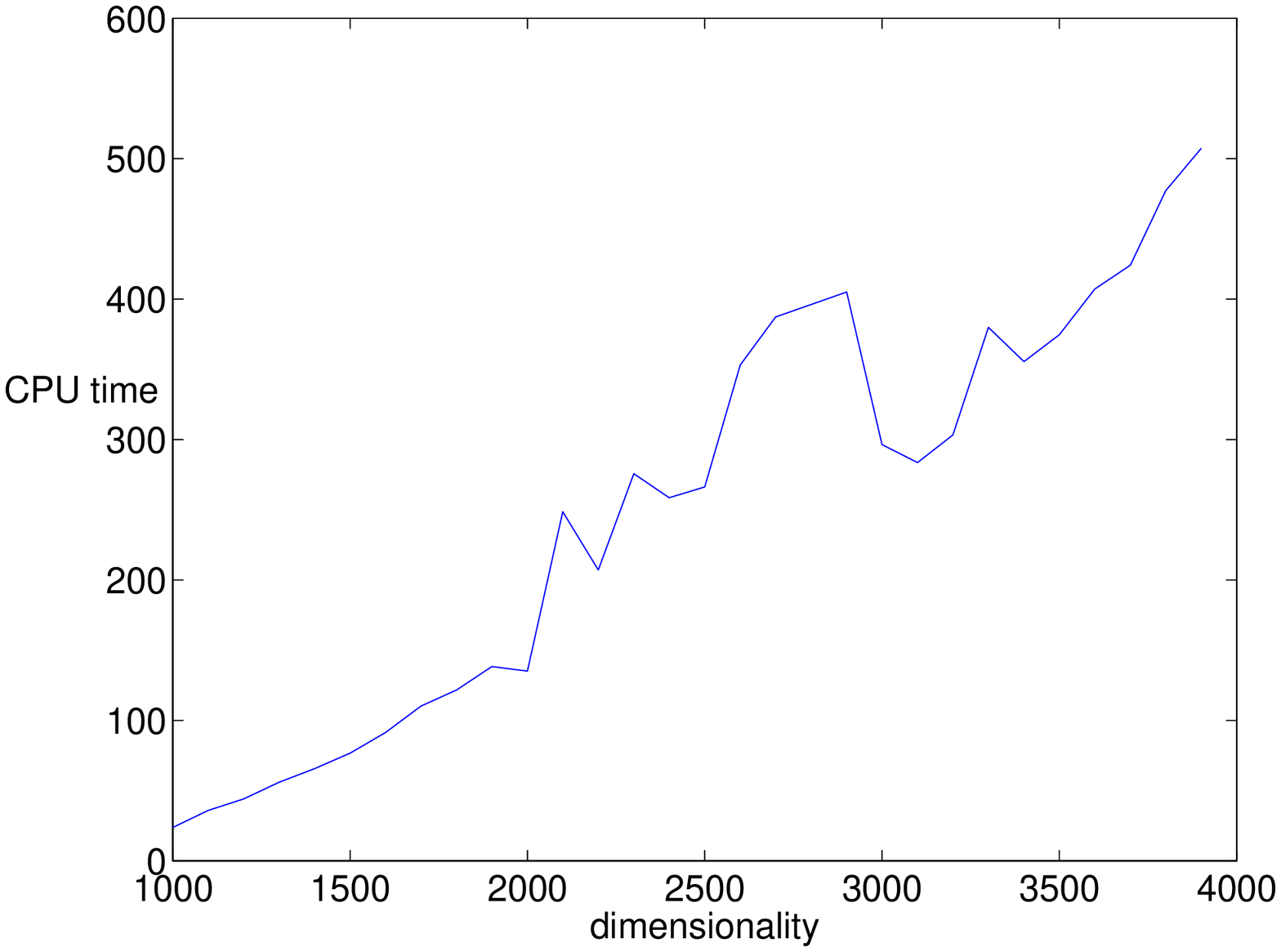}
\caption{Average measures vs. dimensionality for the overlapping groups data.
Top: number of iterations. 
Bottom: CPU time. Note that this time can be reduced to a fraction with a C implementation.}
\label{fig:average}
\end{center}
\end{figure}

\begin{figure}[t]
\begin{center}
\includegraphics[width=0.4\textwidth]{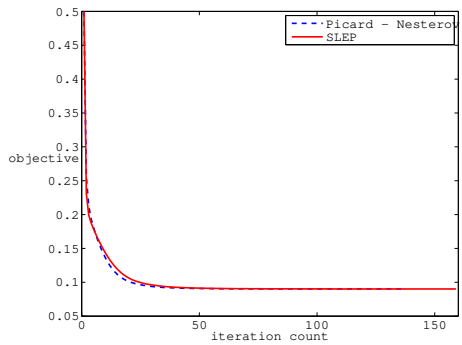}
\caption{Objective function vs. iteration for the hierarchical overlapping groups.}
\label{fig:tree}
\end{center}
\end{figure}

We used a regularization parameter equal to
$10^{-5}$. We ran the algorithm for
$d=1000, 1100, \dots,$ $4000$, with $10$ random data sets 
for each value of $d$, 
and compared its efficiency with SLEP.
%Figure \ref{fig:sol} depicts the solution 
%found for one such simulation ($d=1000$), which recovers the correct pattern
%without exact zeros due to the regularization.
The solutions found recover the correct pattern
without exact zeros due to the regularization.
Figure \ref{fig:costs} shows the number
of iterations $T$ in Algorithm \ref{alg:acc}
needed for convergence in objective value within $\varepsilon=10^{-8}$.
SLEP was run until the same objective value was reached.
% andy
We conclude that we outperform SLEP's $O\left(\frac{1}{T}\right)$
method.  Figure \ref{fig:fp} demonstrates the efficiency of the inner
computation of the proximity map at one iteration $t$ of the
algorithm. Just a few Picard iterations are required for
convergence. The plots for different $t$ are indistinguishable.

Similar conclusions can be drawn from the plots in Figure
\ref{fig:average}, where average counts of iterations
and CPU time are shown for each value of $d$. 
We see that the number of iterations depends almost linearly
on dimensionality and that SLEP requires an order of
magnitude more iterations -- which grow at a higher rate. Note also that the cost
per iteration is comparable between the two methods.
We also observed that computation of the 
proximity map is insensitive to the size of the problem (it only requires
$7-8$ iterations for all $d$).
Finally, we report that CPU time grows linearly with dimensionality.
To remove various overheads this estimate was obtained from 
Matlab's profiling statistics for the low-level functions called.
A comparison with SLEP is meaningless since the latter is
a C implementation.

Besides outperforming the $O(\frac{1}{T})$ method, we also show that
the Picard-Nesterov approach matches SLEP's $O(\frac{1}{T^2})$ method
for the tree structured Group Lasso \cite{tree}. To this end, we have
imitated an experiment from \cite[Sec.~4.1]{hierarchical} using the Berkeley
segmentation data
set\footnote{http://www.eecs.berkeley.edu/Research/Projects/CS/vision/bsds/}.
We have extracted a random dictionary of $71$ $16 \times 16$ patches
from these images, which we have placed on a balanced tree with
branching factors $10, 2, 2$ (top to bottom). Here the groups
correspond to all subtrees of this tree. We have then learned the
decomposition of new test patches in the dictionary basis by Group
Lasso regularization \eqref{eq:GL}. As Figure \ref{fig:tree} shows,
our method and SLEP are practically indistinguishable.

%%%%%%%%%%%%%%%%%%%%%%%%%%%%%%%%%%%%%%%%%%%%%%%%%%%%%%%%%%%%

\subsection{Graph Prediction}
\label{sec:5.2}
The second simulation is on the graph prediction of \cite{mark09}
in the limit of $p=1$ (composite $\ell_1$). We constructed a
synthetic graph of $d$ vertices, $d=100,120, \dots, 360$ with two clusters of equal
size. The edges in each cluster were selected from a uniform draw
with probability $\frac{1}{2}$ and we explicitly connected $d/25$ 
pairs of vertices between the clusters. The labeled data $y$ were
the cluster labels of $s=10$ randomly drawn vertices.
Note that the effective dimensionality of this problem is $O(d^2)$.
At the time of the paper's writing there is not an accelerated method
with software available online which handles a generic graph.

%XXX: maybe say here or at some point in this section 
%that we cannot compare with SLEP here because they 
%cannot handle this problem. or ust say that we are not aware of any 
%first order behond gradient descent which can solve this problem
%
First, we observed that the solution found recovered perfectly
the clustering. 
%(Figure \ref{fig:sol2}). 
Next, we studied the
decay of the objective function for different problem sizes
(Figure \ref{fig:cost2}). We noted a striking difference from
the case of overlapping groups in that convergence now is not
monotone\footnote{
There is no monotonicity guarantee for Nesterov's accelerated
method.} The nature of decay also differs from graph to graph,
with some cases making fast progress very close to the optimal
value but long before eventual convergence. This observation
suggests future modifications of the algorithm which can 
accelerate convergence by a factor. As an indication,
the distance from the optimum was just $2.2\cdot 10^{-6}, 5.4\cdot 10^{-5}, 1.5\cdot 10^{-5}$ 
at iteration $611, 821, 418$ for $d=100,120,140$, respectively.
We verified in this data as well, that Picard iterations converge very fast
(Figure \ref{fig:fp2}). Finally in Table \ref{tab:graph}
we report average iteration numbers and running times.
These prove the feasibility of solving problems
with large matrices $B$ even using a ``quick and dirty'' 
Matlab implementation.

\begin{figure}[t]
\begin{center}
\includegraphics[width=0.4\textwidth]{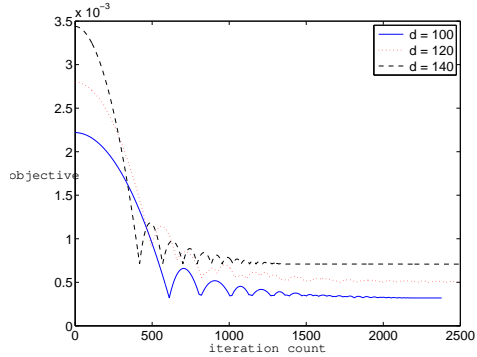}
\caption{Objective function vs. iteration for the graph data.
Note the progress in the early stages in some cases.}
\label{fig:cost2}
\end{center}
\end{figure}

\begin{figure}[t]
\begin{center}
\includegraphics[width=0.4\textwidth]{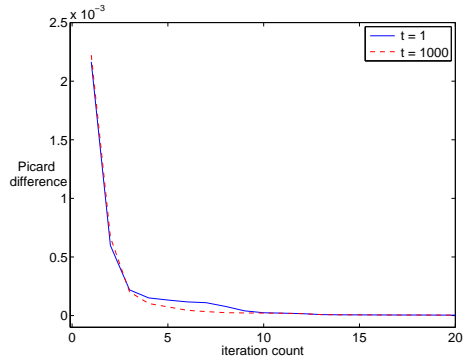}
\caption{$\ell_2$ difference of successive Picard iterates vs. Picard iteration
for the graph data ($d=100$).}
\label{fig:fp2}
\end{center}
\end{figure}

\begin{table}[t]
\label{tab:graph}
\begin{center}
{
\begin{tabular}{|c|c|c|}
\hline $d$ & no. iterations & CPU time (secs.) \\
\hline
100 & 2599.6 & 21.461 \\ 120&                   3680.0 & 54.745\\
140 &                    4351.8 & 118.61\\ 160 &                    3124.8 & 164.21\\
180 &2845.8 & 241.69 \\ 200&                    3476.2 & 359.75 \\
%210 &                   2850.6 & 417.89 \\ 
220&                      4490.0 & 911.67 \\
240 & 4490.0 & 911.67 \\ 260 & 3639.2 & 930.8 \\
\hline
\end{tabular}
}
\vspace{-1.5ex}
\end{center}
\caption{Graph data. Note that the effective $d$ is $O(d^2)$. 
CPU time can be reduced to a fraction with a C implementation.}
\end{table}

\begin{figure}[t]
\begin{center}
\includegraphics[width=0.4\textwidth]{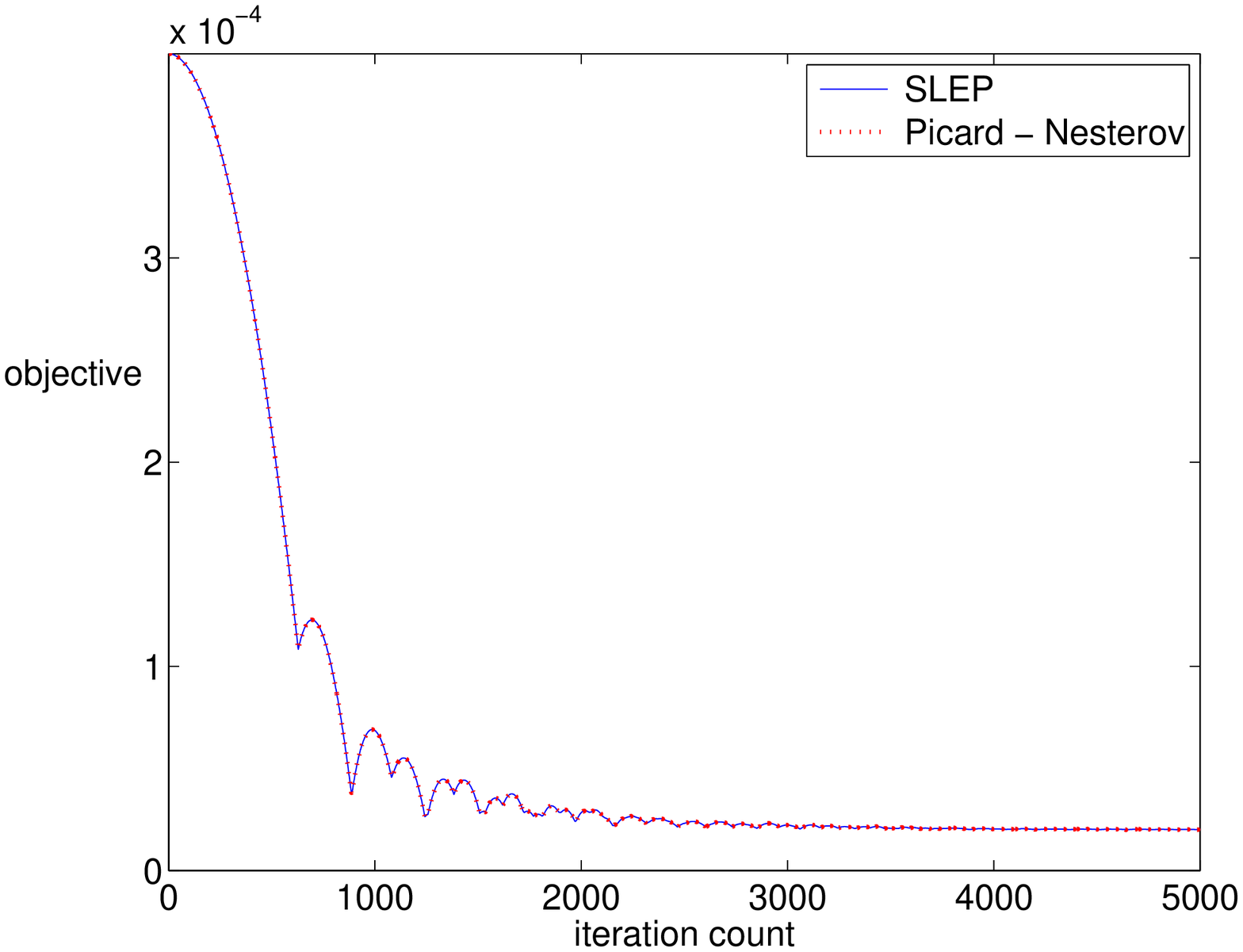}
\caption{Objective function vs. iteration for the Fused Lasso ($d=100$).
The two trajectories are identical.}
\label{fig:fused}
\end{center}
\end{figure}

In addition to a random incidence matrix, one may consider
the special case of {\em Fused Lasso} or {\em Total Variation}
in which $B$ has the simple form \eqref{eq:fused}.  
It has been shown how to achieve the optimal $O\left(\frac{1}{T^2}\right)$
rate for this problem in \cite{beck09}. We applied Fused Lasso (without Lasso regularization) 
to the same clustering data
as before and compared SLEP with the Picard-Nesterov approach.
As Figure \ref{fig:fused} shows, the two trajectories are identical.
This provides even more evidence in favor of optimality of
our method.

%%%%%%%%%%%%%%%%%%%%%%%%%%%%%%%%%%%%%%%%%%%%%%
%%%%%%%%%%%%%%%%%%%%%%%%%%%%%%%%%%%%%%%%%%%%%%

\section{Conclusion}
\label{sec:6}
We presented an efficient first order method for solving a class of
nonsmooth optimization problems, whose objective function is given by
the sum of a smooth term and a nonsmooth term, which is obtained by
linear function composition. The prototypical example covered by this
setting in a linear regression regularization method, in which the
smooth term is an error term and the nonsmooth term is a regularizer
which favors certain desired parameter vectors. An important
feature of our approach is that it can deal with richer classes of
regularizers than current approaches and at the same time is at least
as computationally efficient as specific existing approaches for
structured sparsity. In particular our numerical simulations
demonstrate that the proposed method matches optimal $O(\frac{1}{T^2})$
methods on specific problems (Fused Lasso and tree
structured Group Lasso) while
improving over available $O(\frac{1}{T})$ methods for the overlapping
Group Lasso. In addition, it can handle generic linear composite
regularization problems, for many of which accelerated methods do
not yet exist.  
%maintaining optimal properties of first order methods. 
%Our preliminary experiments indicate that this technique performs well numerically, improving
%over a state of the art optimization method for structured sparsity,
%suggesting that our framework is promising. 
In the future, we wish to study theoretically whether the rate of
convergence is $O\left(\frac{1}{T^2}\right)$, as suggested by our
numerical simulations. There is also much room for further
acceleration of the method in the more challenging cases by using
practical heuristics.  At the same time, it will be valuable to study
further applications of our method. These could include machine
learning problems ranging from multi-task learning, to multiple kernel
learning and to dictionary learning, all of which can be formulated as
linearly composite regularization problems.

%%%%%%%%%%%%%%%%%%%%%%%%%%%%%%%%%%%%%%%%%%%%%%%%%%%%%%%%%%%%

\subsubsection*{Acknowledgements}
We wish to thank Luca Baldassarre and Silvia Villa for 
useful discussions.
This work was supported by Air Force Grant AFOSR-FA9550, EPSRC Grants
EP/D071542/1 and EP/H027203/1, NSF Grant ITR-0312113, Royal Society
International Joint Project Grant 2012/R2, as well as by the IST
Programme of the European Community, under the PASCAL Network of
Excellence, IST-2002-506778.

%%%%%%%%%%%%%%%%%%%%%%%%%%%%%%%%%%%%%%%%%%%%%%%%%%%%%%%%%%%%

\section{Appendix}

In this appendix, we collect some basic facts about fixed point theory
which are useful for our study. For more information on the material presented here, 
we refer the reader to \cite{zalinescu}. 

Let $X$ be a closed subset of $\R^d$. A mapping $\vp: X \rightarrow X$
is called strictly non-expansive (or contractive) if there exists
$\lam \in [0,1)$ such that, for every $x,y\in X$,
$$
\|\vp(x)-\vp(y)\|\leq \lam \|x-y\|.
$$ 
If the above inequality holds for $\lambda =1$, the mapping is called
nonexpansive.  We say that $x$ is a {\em fixed point} of $\vp$ if
$x=\vp(x)$.  The Picard iterates $x^n, n \in \N$ starting at $x_0 \in
X$ are defined by the recursive equation $x^{n} = \vp(x^{n-1})$.

It is a well-knwon fact that, if $\vp$ is strictly nonexpansive then
$\vp$ has a unique fixed point $x$ and $\lim_{n\rightarrow \infty} x^n
= x$.  However, this result fails if $\vp$ is nonexpansive. For
example, the map $\vp(x) = x+1$ does not have a fixed point. On the
other hand, the identity map has infinitely many fixed points.

%\begin{lemma}
%If the maps $\vp_1,\vp_2: X \rightarrow X$ are nonexpansive the map
%defined, for every $\ka \in [0,1]$ as $\ka \vp_1+(1-k)\vp_2$
%, is as well
%nonexpansive.
%\label{lem:a1}
%\end{lemma}
%\begin{proof}
%Let $\vp = \ka \vp_1(x) +(1-\ka)\vp_2$. Note that 
%\begin{eqnarray}
%\|\vp(x)-\vp(y)\|
%& = & \|\ka(\vp_1(x)-\vp_1(y)) + (1-\ka)(\vp_2(x)-\vp_2(y))\| \\ \nonumber
%~&\leq & \ka \| \vp_1(x)-\vp_1(y)\| + (1-\ka)\|\vp_2(x)-\vp_2(y) \|
%\leq \|x-y\|.
%\end{eqnarray}
%\end{proof}
%\begin{corollary}
%If $\vp$ is nonexpansive then, for every $\kappa \in [0,1]$ the
%mapping $k I + (1-\kappa) \vp$ is nonexpansive and has the same set
%of fixed points of $\vp$.
%\label{cor:a1}
%\end{corollary}
%\begin{lemma}
%If $\vp$ is firmly nonexpansive then $a\vp - I$ is nonexpansive for every $a \in [0,2]$.
%\label{lem:a2}
%\end{lemma}
%\begin{proof}
%Using Cauchy-Swartz's inequality we have that
%\begin{eqnarray}
%\|(a \vp(x) -x) - (a\vp(y)-y)\|^2
%& = & \|x-y\|^2+a^2\|\vp(x)-\vp(y)\|^2 - 2a \la\vp(x)-\vp(y),x-y\ra\\ \nonumber
%~&\leq & \|x-y\|^2 + a(a-2) \|\vp(x)-\vp(y)\|^2 \leq \|x-y\|^2.
%\end{eqnarray}
%\end{proof}
%Note that $\vp = \frac{1}{2} (2\vp - I) + \frac{1}{2} I$. Thus if $\vp$ is 
%firmly nonexpansive, we conclude by Lemma \ref{lem:a1} and Lemma \ref{lem:a2} with $a=2$ that $\vp$ is a strict convex combination of two nonexpansive maps.

\begin{definition}
Let $X$ be a closed subset of $\R^d$. A map $\vp: X \rightarrow X$ is 
called asymptotically regular provided that $\lim_{n \rightarrow \infty} \|x^{n+1}-x^n\|=0$.
\end{definition}

\begin{proposition}
Let $X$ be a closed subset of $\R^d$ and $\vp: X \rightarrow X$ such that 
\begin{enumerate}
\item $\vp$ is nonexpansive;
\item $\vp$ has at least one fixed point;
\item $\vp$ is asymptotically regular.
\end{enumerate}
Then the sequence $\{x^n: n \in \N\}$ converges to a fixed point of $\vp$.
\label{prop:1}
\end{proposition}
\begin{proof}
We divide the proof in three steps. 

{\em Step 1:} The Picard iterates are bounded. Indeed, let $x$ be a fixed point of $\vp$. We have that
$$
\|x^{n+1}-x\| = \|\vp(x^{n})-\vp(x)\| \leq \|x^{n}-x\| \leq \cdots \leq \|x^0-x\|.
$$ 

{\em Step 2:} Let $\{x^{n_k}: k \in \N\}$ be a convergent
subsequence, whose limit we denote by $y$. We will show that $y$ is a
fixed point of $\vp$. Since $\vp$ is continuous, we have that
$\lim_{k \rightarrow \infty} (x^{n_k}-\vp(x^{n_k})) = y - \vp(y)$, and
since $\vp$ is asymptotically regular $y-\vp(y) = 0$.  

{\em Step 3:} The
whole sequence converges. Indeed, following the same reasoning in the
proof of Step 1, we conclude that the sequence $\{\|x^{n}-y\|: n \in
\N\}$ is non-increasing. Let $\alpha = \lim_{n\rightarrow \infty}
\|x^n-y\|$. Since $\lim_{k\rightarrow \infty} \|x^{n_k}-y\| = 0$,
we conclude that $\alpha=0$ and, so, $\lim_{n \rightarrow \infty} x^n
= y$.
\end{proof}
We note that in general, without the asymptotically regularity
assumption, the Picard iterates do not converge. For example,
consider $\vp(x) = -x$. Its only fixed point is $x=0$; if we start
from $x^0 \neq 0$ the Picard iterates will oscillate. Moreover, if
$\vp(x) = x+1$, which is nonexpansive, the Picard iterates diverge.

We now discuss the main tool which we use to find a fixed point of a
nonexpansive mapping $\vp$.
\begin{theorem} (Opial $\ka$-average theorem \cite{opial})
Let $X$ be a closed convex subset of~$\R^d$, $\vp: X \rightarrow X$ a
nonexpansive mapping, which has at least one fixed point and let
$\vp_\ka := \ka I + (1-\ka) \vp$.  Then, for every $\ka \in (0,1)$,
the Picard iterates of $\vp_\ka$ converge to a fixed point of $\vp$.
\label{thm:opial_app}
\end{theorem}
We prepare for the proof with two useful lemmas.

\begin{lemma}
If $\ka \in (0,1)$, $u,w \in \R^d$, $\|u\|\leq \|w\|$, then 
$$
\kappa(1-\kappa) \|w-u\|^2 \leq \|w\|^2 - \|\ka w + (1-\ka) u\|^2
$$
\label{lem:A}
\end{lemma}
\begin{proof}
The assertion follows from $\ell_2$ strong convexity,
\begin{align*}
\kappa(1-\kappa) \|w-u\|^2 + \|\ka w + (1-\ka) u\|^2 \\= 
\ka \|w\|^2 +(1-\ka) \|u\|^2 
\leq \|w\|^2 \,.
\end{align*} 
\end{proof}

%\begin{lemma}
%If $\ka \in (0,1)$, $\alpha \in (0,\min\{\ka,1-\ka\})$, $u,w \in \R^d$, $\|u\|\leq \|w\|$, then 
%$$
%\alpha^2 \|w-u\|^2 \leq \|w\|^2 - \|\ka w + (1-\ka) u\|^2
%$$
%\label{lem:A}
%\end{lemma}

%\begin{proof}
%We note that 
%\begin{eqnarray}
%&& \|\ka w +(1-\ka)u\|^2 + \alpha^2\|w-u\|^2 =  \\ \nonumber
%&&\frac{1}{2}\| (\ka+\alpha)w+(1-\ka-\alpha)u\|^2 \\ \nonumber
%&& + \frac{1}{2}\|(\ka-\alpha)w + (1-\ka+\alpha)u\|^2.
%\end{eqnarray}
%Using the above inequality and the triangle inequality we have that
%\begin{eqnarray}
%\alpha^2 \|w-u\|^2 & \leq & \frac{1}{2} \Big[(\ka +\alpha) \|w\|
% + (1-\ka -\alpha)\|u\|\Big]^2 \\ \nonumber
%~&~& + \frac{1}{2} \Big[(\ka-\alpha) \|w\| + (1-\ka+\alpha)\|u\|\Big]^2 \\ \nonumber
%&&- \|\ka w +(1-\ka)u\|^2 \\ \nonumber
%~ & \leq & \|w\|^2 - \|\ka w +(1-\ka)u\|^2.
%\end{eqnarray}
%\end{proof}

\begin{lemma} 
If $\{u^n : n \in \N\}$ and $\{w^n: n \in \N\}$ are sequences in $\R^d$ 
such that $\lim_{n\rightarrow \infty}\|w^n\| = 1$, $\|u^n\|\leq \|w^n\|$ and
$\lim_{n \rightarrow \infty} \|\ka w^n + (1-\ka)u^n\| = 1$, 
then $\lim_{n\rightarrow \infty} w^n-u^n = 0$.
\label{lem:B}
\end{lemma}
\begin{proof}
Apply Lemma \ref{lem:A} to note that
$$
\ka(1-\ka) \|w^n - u^n\|^2 \leq \|w^n\|^2 - \|\ka w^n +(1-\ka)u^n\|^2.
$$
By hypothesis the right hand side tends to zero as $n$ tends to infinity and the result follows.
\end{proof}

\begin{proof}[{\bf Proof of Theorem \ref{thm:opial_app}}]
Let $\{x^n: n \in \N\}$ be the iterates of $\vp_\ka$. We will show that
$\vp_\ka$ is asymptotically regular. The result will then follow by
Proposition \ref{prop:1} and the fact that $\vp_\ka$ and $\vp$ have
the same set of fixed points.

Let $x^{n+1} = \ka x^n + (1-\ka) \vp(x^n)$.
Note that, if $u$ is fixed point of $\vp_\ka$, then 
$$
\|x^{n+1}-u\| \leq \|x^n - u\| \leq \cdots \leq \|x^0-u\| \,.
$$
Let $\bar{d} := \lim_{n\rightarrow \infty} \|x^n-u\|$. If $\bar{d}=0$ the result is proved.
We will show that if $\bar{d}>0$ we contradict the hypotheses of the theorem. 
For every $n \in \N$, we define 
$w^n = \bar{d}^{-1}(x^n-u)$ and $u^n =  \bar{d}^{-1}(\vp(x^n)-u)$. 
Note that the sequences $\{w^n: n \in \N\}$ and 
$\{u^n:n \in \N\}$ satisfy the hypotheses of Lemma \ref{lem:B}. Thus, we have that 
$\lim_{n\rightarrow \infty} (x^n-\vp(x^n))= 0$. Consequently
$x^{n+1}-x^n = (1-\ka) (\vp(x^n)-x^n) \rightarrow 0$,  
%$$
%x^{n+1}-x^n = 
%%\ka x^n+(1-\ka)(\vp(x^n)-x^n= 
%(1-\ka) (\vp(x^n)-x^n) \rightarrow 0
%$$
showing that $\{x^n:n \in \N\}$ is asymptotically regular.
\end{proof}

\end{document}